\documentclass[nohyperref]{article}

\usepackage{graphicx}
\usepackage{subfigure}
\usepackage{booktabs} 
\usepackage{multirow}

\usepackage[hidelinks]{hyperref}

\usepackage[accepted]{icml2022}

\usepackage{amsthm}
\usepackage{amssymb}
\usepackage{amsmath}
\usepackage{bbm}
\usepackage{bm}
\usepackage{mathtools}
\usepackage{cleveref}
\usepackage{xfrac}
\usepackage{enumitem}
\usepackage{array}
\usepackage{xcolor}
\newcolumntype{H}{>{\setbox0=\hbox\bgroup}c<{\egroup}@{}}

\theoremstyle{plain}
\newtheorem{theorem}{Theorem}

\newtheorem{definition}[theorem]{Definition}

\newcommand{\R}{\mathbb{R}}
\newcommand{\dist}{\mathcal{D}}
\DeclareMathOperator*{\Exp}{\mathbb{E}}

\newcommand{\clf}{f_{\texttt{clf}}}
\newcommand{\fdet}{f_{\texttt{det}}}



\icmltitlerunning{Detecting Adversarial Examples Is (Nearly) As Hard As Classifying Them}

\begin{document}

\twocolumn[
\icmltitle{Detecting Adversarial Examples Is (Nearly) As Hard As Classifying Them}

\begin{icmlauthorlist}
	\icmlauthor{Florian Tramèr}{g,s}
\end{icmlauthorlist}

\icmlaffiliation{s}{Work done while the author was at Stanford University}
\icmlaffiliation{g}{Google Research}

\icmlcorrespondingauthor{Florian Tramèr}{tramer@cs.stanford.edu}

\icmlkeywords{Adversarial examples, detection, hardness}

\vskip 0.3in
]

\printAffiliationsAndNotice{}

\begin{abstract}
Making classifiers robust to adversarial examples is challenging. 
Thus, many works tackle the seemingly easier task of \emph{detecting} perturbed inputs.

We show a barrier towards this goal. We prove a \emph{hardness reduction} between detection and classification of adversarial examples: given a robust detector for attacks at distance $\epsilon$ (in some metric), we show how to build a similarly robust (but computationally inefficient) \emph{classifier} for attacks at distance $\epsilon/2$.

Our reduction is \emph{computationally inefficient}, but preserves the \emph{sample complexity} of the original detector. The reduction thus cannot be directly used to build practical classifiers.
Instead, it is a useful sanity check to test whether empirical detection results imply something much stronger than the authors presumably anticipated (namely a highly robust and data-efficient \emph{classifier}).

To illustrate, we revisit $14$ empirical detector defenses published over the past years. For $12/14$ defenses, we show that the claimed detection results imply an inefficient classifier with robustness far beyond the state-of-the-art.


\end{abstract}

\section{Introduction}
Building models that are robust to adversarial examples~\citep{szegedy2013intriguing, biggio2013evasion} is a major challenge and open problem in machine learning.
Due to the inherent difficulty in building robust \emph{classifiers}, researchers have attempted to build techniques to at least \emph{detect} adversarial examples, a weaker task that is largely considered easier than robust classification~\citep{xu2017feature, pang2021adversarial, sheikholeslami2021provably}.

Yet, evaluating the robustness of detector defenses is challenging. This is in part due to a lack of strong evaluation guidelines and benchmarks---akin to those developed for robust classifiers~\citep{carlini2019evaluating, croce2020robustbench}---as well as to a lack of long-standing comparative baselines such as adversarial training~\citep{madry2018towards}.
As a result, when a new detection defense is proposed, it can be hard to judge if the claimed empirical results are credible.

To illustrate this challenge, consider the following (fictitious) claims about two defenses against adversarial examples on CIFAR-10: 
\begin{itemize}
	\item defense A is a classifier that claims robust accuracy of $90\%$ under $\ell_\infty$-perturbations bounded by $\epsilon=\sfrac{\bf{4}}{255}$; 
	\item defense B also has a ``rejection'' option, and claims robust accuracy of $90\%$ under $\ell_\infty$-perturbations bounded by $\epsilon=\sfrac{\bf{8}}{255}$ (we say that defense B is robust for some example if it classifies that example correctly, and either rejects/detects or correctly classifies all perturbed examples at distance $\epsilon$.)
\end{itemize}

\emph{Which of these two (empirical) claims are you more likely to believe to be correct?}

Defense A claims much higher robustness than the current best results achieved with adversarial training~\citep{madry2018towards, rebuffi2021fixing}, the only empirical defense against adversarial examples that has stood the test of time. Indeed, the state-of-the-art $\ell_\infty$ robustness for $\epsilon=\sfrac{4}{255}$ on CIFAR-10 (without external data) is $\approx 79\%$~\citep{rebuffi2021fixing}.
Thus, the claim of defense A is likely to be met with initial skepticism and heightened scrutiny, as could be expected for such a claimed breakthrough result.

The claim of defense B is harder to assess, due to a lack of long-standing baselines for robust detectors (many detection defenses have been shown to be broken~\citep{carlini2017adversarial, tramer2020adaptive}). 
On one hand, detection of adversarial examples has largely been considered to be an easier task than classification~\citep{xu2017feature, pang2021adversarial, sheikholeslami2021provably}. On the other hand, defense B claims robustness to perturbations that are twice as large as defense A ($\epsilon=\sfrac{\bf{8}}{255}$ vs. $\epsilon=\sfrac{\bf{4}}{255}$).

\emph{In this paper, we show that the claims of defenses A and B are, in fact, equivalent! (up to computational efficiency.)}

We prove a general \emph{hardness reduction} between classification and detection of adversarial examples. Given a detector defense that achieves robust risk $\alpha$ for attacks at distance $\epsilon$ (under any metric), we show how to build an \emph{explicit but inefficient} classifier that achieves robust risk $\alpha$ for classifying attacks at distance $\epsilon/2$. The reverse implication also holds: a classifier robust at distance $\epsilon/2$ implies an explicit but inefficient robust detector at distance $\epsilon$. Thus, we prove that robust classification and detection are \emph{equivalent}, up to computational complexity and a factor $2$ in the robustness radius.

The reader might argue that a computationally inefficient reduction may be of limited use. That is, maybe we should not be surprised by the possibility of a breakthrough in robust classification that leverages inefficient methods. Yet, we do not yet know of any way of leveraging computational \emph{inefficiency} to build more robust models. To the contrary, there is growing evidence that the robustness of classifiers is limited by the availability of \emph{data}~\cite{schmidt2018adversarially, carmon2019unlabeled, uesato2019labels, zhai2019adversarially} rather than by \emph{compute}.
Since our reduction does preserve the \emph{sample complexity} of the original detector model, 
we thus argue that---given our current knowledge---we should be as ``surprised'' by the claim made by defense B above as by the claim made by defense A.

Our reduction thus provides a way of assessing the ``plausibility'' of new robust detection claims, by contrasting them with results from the more mature literature on robust classification.
To illustrate, we revisit 14 published detection defenses across three datasets, and show that in 12/14 cases the defense's robust detection claims actually imply a classifier with robustness far superior to the current state-of-the-art. Yet, none of these detection papers make the claim that their techniques should imply such a breakthrough in robust \emph{classification}. 

This situation is reminiscent of an informal principle laid out by Scott Aaronson to recognize likely-to-be-incorrect mathematical breakthroughs: 
\begin{quote}
``\emph{The approach seems to yield something much stronger and maybe even false (but the authors never discuss that).}''~\citep{aaronson}
\end{quote}
Using our reduction, it is obvious that many detector defenses are implicitly claiming much stronger robustness than we believe feasible with current techniques. And indeed, subsequent re-evaluations of many of these defenses resulted in complete breaks~\citep{carlini2017adversarial, tramer2020adaptive}. 

Remarkably, we find that for \emph{certified} defenses, the state-of-the-art results for provable robust classification and detection perfectly match the results implied by our reduction. For example, \citet{sheikholeslami2021provably} recently proposed a certified detector on CIFAR-10 with provable robust error that is within $3\%$ of the provable error of the \emph{inefficient} detector obtained by combining our result with
the state-of-the-art robust classifier of \citet{zhang2019towards}.

In summary, we prove that giving classifiers access to a detection option does not help robustness (or at least, not much). Our work provides, to our knowledge, the first example of a hardness reduction between different approaches for robust machine learning. As in the case of computational complexity, we believe that such reductions can be useful for identifying research questions or areas that are unlikely to bear fruit (bar a significant breakthrough)---so that the majority of the community's efforts can be redirected elsewhere.

On a technical level, our reduction exposes a natural connection between robustness and error correcting codes, which may be of independent interest.

\section{Hardness Reductions Between Robust Classifiers and Detectors}
\label{sec:theory}

In this section, we prove our main result: a reduction between robust detectors and robust classifiers. We first introduce some useful notation and define the (robust) risk of classifiers with and without a detection option.

\subsection{Preliminaries}
We consider a classification task with a distribution $\dist$ over examples $x \in \R^n$ with labels $y \in [C]$. 
A classifier is a function $\clf: \R^n \to [C]$.
A detector $\fdet$ is a classifier with an extra ``rejection'' or "detection" option $\bot$, that indicates the absence of a classification.
We assume for simplicity that classifiers and detectors are deterministic. Our results can easily be extended to randomized functions as well.
The binary indicator function $\mathbbm{1}_{\{A\}}$ is $1$ if and only if the predicate $A$ is true.

We first define a classifier's \emph{risk}, i.e., its classification error on unperturbed samples.

\begin{definition}[Risk]
	Let $f$ be either a classifier $\clf: \mathbb{R}^n \to [C]$ or a classifier with detection $\fdet: \mathbb{R}^n \to [C] \cup \{\bot\}$. The risk of $f$ is the expected rate at which $f$ fails to correctly classify a sample:
	\begin{equation}
		R(f) \coloneqq \Exp_{(x, y) \sim \dist} \left[\mathbbm{1}_{\left\{f(x) \neq y\right\}}\right]
	\end{equation}
\end{definition}

Note that for a detector, rejecting an unperturbed example sampled from the distribution $\dist$ is counted as an error.

For classifiers without a rejection option, we define the \emph{robust risk} as the risk on worst-case adversarial inputs~\citep{madry2018towards}. For an input $x$ sampled from $\dist$, an adversarial example $\hat{x}$ is constrained to being within distance $d(x, \hat{x}) \leq \epsilon$ from $x$, where $d$ is some metric.

\begin{definition}[Robust risk]
	Let $\clf: \mathbb{R}^n \to [C]$ be a classifier.
	The robust risk at distance $\epsilon$ is:
	\begin{equation}
		R_{\text{adv}}^\epsilon(\clf) \coloneqq \Exp_{(x, y) \sim \dist} \left[ \max_{d(x, \hat{x}) \leq \epsilon} \mathbbm{1}_{\left\{\clf(\hat{x}) \neq y\right\}}\right]
	\end{equation}
\end{definition}

Thus, a sample $(x, y)$ is robustly classified if and only if every point within distance $\epsilon$ of $x$ (including $x$ itself) is correctly classified as $y$.

For a detector (a classifier with an extra detection/rejection output), we similarly define the robust risk with detection. The classifier is now allowed to reject adversarial inputs.

\begin{definition}[Robust risk with detection]
	\label{def:risk_det}
	Let $\fdet: \mathbb{R}^n \to [C] \cup \{\bot\}$ be a classifier with an extra detection output $\bot$.
	The robust risk with detection at distance $\epsilon$ is:
	\begin{align}
	&R_{\text{adv-det}}^\epsilon(\fdet) \coloneqq \nonumber\\ 
	&\quad\quad\Exp_{(x, y) \sim \dist} \left[ \max_{d(x, \hat{x}) \leq \epsilon} \mathbbm{1}_{\left\{\fdet(x) \neq y \ \lor\ \fdet(\hat{x}) \notin \{y, \bot\}\right\}}\right]\hspace{-3pt}
	\end{align}
\end{definition}

That is, a detector defense $\fdet$ is robust on a sample $x$ from the distribution if and only if the defense classifies the clean input $x$ correctly, and the defense either rejects or correctly classifies every perturbed input $\hat{x}$ within distance $\epsilon$ from $x$. The requirement that the defense correctly classify unperturbed examples eliminates pathological defenses that reject all inputs.

\subsection{Robust Detection Implies Inefficient Robust Classification}
We are now ready to introduce our main result, a reduction from a robust detector for adversarial examples at distance $\epsilon$, to an inefficient robust classifier at distance $\epsilon/2$. We later prove that this reduction also holds in the reverse direction, thereby demonstrating the equivalence between robust detection and classification---up to computational hardness.

\begin{theorem}[$\epsilon$-robust detection implies inefficient $\epsilon/2$-robust classification]
	\label{thm:det_to_cls}
	Let $d(\cdot, \cdot)$ be an arbitrary metric.
	Let $\fdet$ be a detector that achieves risk $R(\fdet)=\alpha$, and robust risk with detection $R_{\text{adv-det}}^\epsilon(\fdet) = \beta$. Then, we can construct an explicit (but inefficient) classifier $\clf$ that achieves risk $R(\clf) \leq \alpha$ and robust risk $R_{\text{adv}}^{\epsilon/2}(\clf) \leq \beta$.
	
	The classifier $\clf$ is constructed as follows on a (possibly perturbed) input $\hat{x}$:
	\begin{itemize}
		\item Run the detector model $y \gets \fdet(\hat{x})$. If the input is not rejected, i.e., $y \neq \bot$, then output the label $y$ that was predicted by the detector.
		\item Otherwise, find an input $x'$ within distance $\epsilon/2$ of $\hat{x}$ that is not rejected, i.e., $d(\hat{x}, x') \leq  \epsilon/2$ and $\fdet(x') \neq \bot$. If such an input $x'$ exists, output the label $y \gets \fdet(x')$. Else, output a uniformly random label $y \in [C]$.
	\end{itemize}
	\vspace{0.25em}
\end{theorem}

An intuitive illustration for our construction, and for the proof of the theorem (see below) is in \Cref{fig:intuition}. 

Our construction can be viewed as an analog of \emph{minimum distance decoding} in coding theory. We can view a clean data point $x$ sampled from $\dist$ as a codeword, and an adversarial example $\hat{x}$ as a noisy message with a certain number of errors (where the error magnitude is measured using an arbitrary metric on $\R^n$ rather than the Hamming distance that is typically used for error correcting codes). A standard result in coding theory states that if a code can \emph{detect} $\alpha$ errors, then it can \emph{correct} $\alpha/2$ errors. This result follows from a ``ball-packing'' argument: if $\alpha$ errors can be detected, then any two valid codewords must be at least at distance $\alpha$ from each other, and therefore $\alpha/2$ errors can be corrected via minimum distance decoding.

\begin{figure}[t]
	\centering
	\includegraphics[width=0.8\columnwidth]{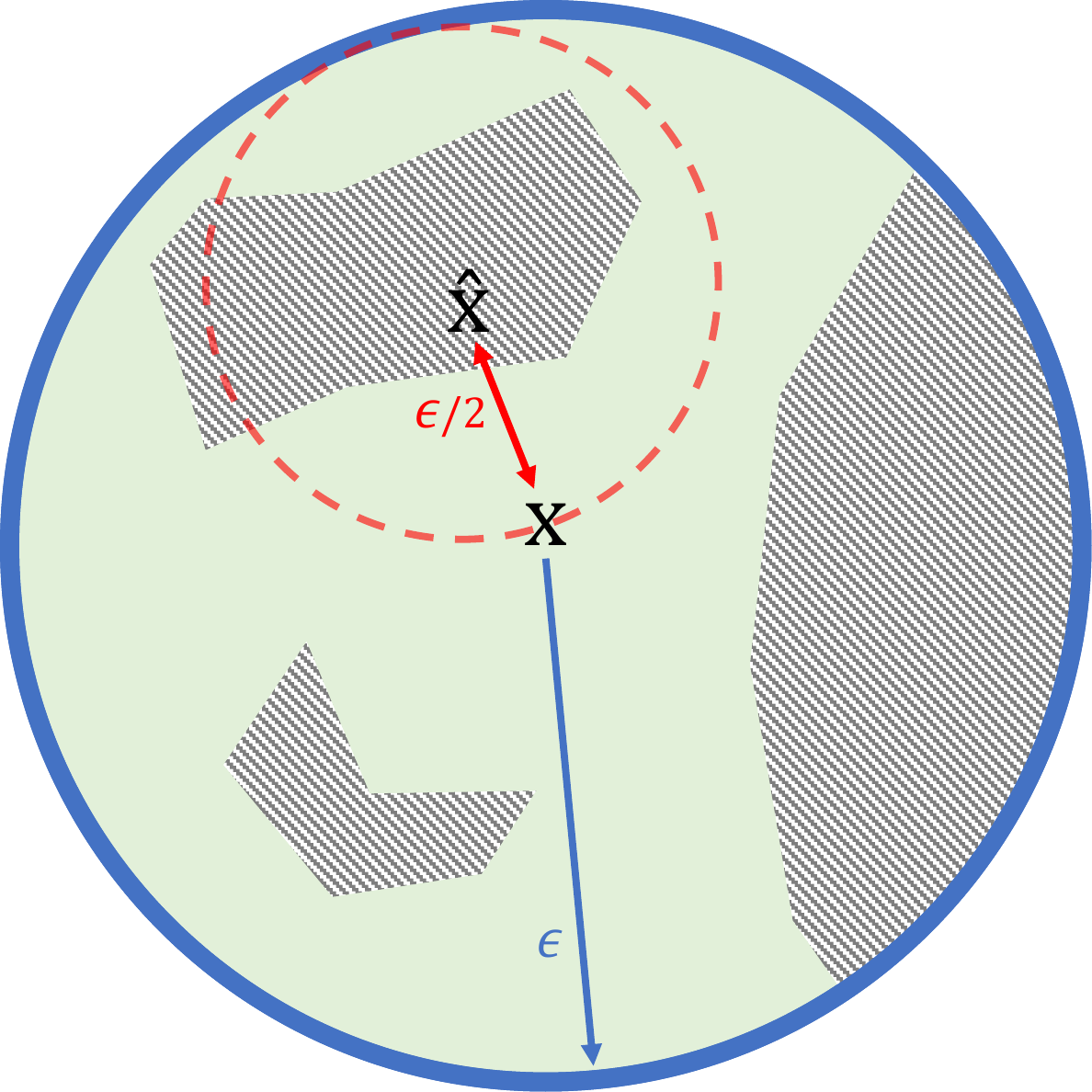}
	\caption{Illustration of the construction of a robust classifier from a robust detector in \Cref{thm:det_to_cls}. The outer blue circle represents all inputs at distance $\epsilon$ from the input $x$. For a detector $\fdet$, the areas in green correspond to correctly classified inputs, and hatched gray areas correspond to rejected inputs. The detector $\fdet$ is thus robust on $x$ up to distance $\epsilon$. The classifier $\clf$ classifies a perturbed input $\hat{x}$, at distance $\epsilon/2$ from $x$, by finding any input within distance  $\epsilon/2$ from $\hat{x}$ (the red dashed circle) that is not rejected by $\fdet$. Such an input necessarily exists and is correctly labeled by $\fdet$. The classifier $\clf$ is thus robust on $x$ up to distance $\epsilon/2$.}
	\label{fig:intuition}
\end{figure}

\begin{proof}[Proof of \Cref{thm:det_to_cls}]
	First, note that the natural accuracy of our constructed classifier $\clf$ is at least as high as that of the detector $\fdet$, since $\clf$ mimics the output of $\fdet$ whenever $\fdet$ does not reject an input sampled from $\dist$. Thus, $R(\clf) \leq R(\fdet) = \alpha$.
	
	Now, for the sake of contradiction, consider an input $(x, y) \sim \dist$ for which the constructed classifier $\clf$ is not robust at distance $\epsilon/2$. By construction, this means that there exists some input $\hat{x}$ at distance $\epsilon/2$ from $x$ such that $\hat{x}$ is misclassified, i.e., $\clf(\hat{x}) = \hat{y} \neq y$. We will show that the detector $\fdet$ is not robust with detection for $x$ either (for attacks at distance up to $\epsilon$).
	
	By definition of the classifier $\clf$, if $\clf(\hat{x}) = \hat{y} \neq y$ then either:
	\begin{itemize}[leftmargin=12pt]
		\item The detector $\fdet$ also misclassifies $\hat{x}$, i.e., $\fdet(\hat{x}) = \hat{y}$.
		
		So $\fdet$ is not robust with detection for $x$ at distance $\epsilon$.
		
		\item There exists an input $x'$ within distance $\epsilon/2$ of $\hat{x}$, such that the detector $\fdet$ misclassifies $x'$, i.e. $\fdet(x') = \hat{y}$.
		
		Note that by the triangular inequality, $d(x, x') \leq d(x, \hat{x}) + d(\hat{x}, x') \leq \epsilon/2 + \epsilon/2 = \epsilon$, and thus $\fdet$ is not robust with detection for $x$ at distance $\epsilon$.
		
		\item The detector $\fdet$ rejects all inputs $x'$ within distance $\epsilon/2$ of $\hat{x}$ (and thus $\clf$ sampled an output label at random).
		
		Since $d(x, \hat{x}) \leq \epsilon/2$, this implies that the detector also rejects the clean input $x$, i.e., $\fdet(x) = \bot$, and thus $\fdet$ is not robust with detection for $x$.
	\end{itemize}

	In summary, whenever the constructed classifier $\clf$ fails to robustly classify an input $x$ up to distance $\epsilon/2$, the detector $\fdet$ also fails to robustly classify $x$ with detection up to distance $\epsilon$.
	Taking expectations over the entire distribution $\dist$ concludes the proof.
\end{proof}

The classifier $\clf$ constructed in \Cref{thm:det_to_cls} is computationally inefficient. Indeed, the second step of the defense consists in finding a non-rejected input within some metric ball. If the original detector $\fdet$ is a non-convex function (e.g., a deep neural network), then this step consists in solving an intractable non-convex optimization problem.
Our reduction is thus typically not suitable for building a practical robust classifier. Instead, it demonstrates the existence of an inefficient but \emph{explicitly instantiable} robust classifier. We discuss the implications of this result more thoroughly in \Cref{sec:experiments}.

A corollary to our reduction is that many ``information theoretic'' results about robust classifiers can be directly extended to robust detectors. For example, \citet{tsipras2019robustness} prove that there exists a formal tradeoff between a classifier's clean accuracy and robust accuracy for certain natural tasks. Since their result applies to \emph{any} classifier (including inefficient ones), combining their result with our reduction implies that a similar accuracy-robustness tradeoff exists for detectors.
More precisely, \citet{tsipras2019robustness} show that for certain classification tasks and suitable choices of parameters $\alpha, \beta, \epsilon$, any classifier $\clf$ which achieves risk $R(\clf) \leq \alpha$  must have robust risk at least $R_{\text{adv}}^\epsilon(\clf) \geq \beta$ against $\ell_\infty$-perturbations bounded by $\epsilon$. By our reduction, this implies that any detector $\fdet$ with risk at most $R(\fdet) \leq \alpha$ must also have robust risk with detection at least $R_{\text{adv-det}}^{\epsilon/2}(\fdet) \geq \beta$ against $\ell_\infty$-perturbations bounded by $\epsilon/2$.

Similar arguments can be applied to show, for instance, that the increased sample complexity of robust generalization from~\citet{schmidt2018adversarially}, or the tradeoff between robustness to multiple perturbation types from~\citet{tramer2019adversarial}, also apply to robust detectors.

Our reduction does not apply to ``computational'' hardness results that have been shown for robust classification. For example, \citet{garg2020adversarially} and \citet{bubeck2018adversarial} show (``unnatural'') distributions where learning a robust classifier is computationally hard---under standard cryptographic assumptions. We cannot use \Cref{thm:det_to_cls} to conclude that learning a robust \emph{detector} is hard for these distributions, since the existence of such a detector would only imply an inefficient robust classifier which does not contradict the results of \citet{garg2020adversarially} or \citet{bubeck2018adversarial}.

\subsection{Robust Classification Implies Inefficient Robust Detection}
A similar argument to the proof of \Cref{thm:det_to_cls} can be used in the opposite direction, to show that a robust classifier at distance $\epsilon/2$ implies an inefficient robust detector at distance $\epsilon$.
Taken together, \Cref{thm:det_to_cls} and \Cref{thm:cls_to_det} show that robust detection and classification are \emph{equivalent}, up to a factor $2$ in the norm bound and up to computational constraints.

\begin{theorem}[$\epsilon/2$ robust-classification implies inefficient $\epsilon$-robust detection]
	\label{thm:cls_to_det}
	Let $d(\cdot, \cdot)$ be an arbitrary metric.
	Let $\clf$ be a classifier that achieves robust risk $R_{\text{adv}}^{\epsilon/2}(\clf) = \beta$. Then, we can construct an explicit (but inefficient) detector $\fdet$ that achieves risk $R(\fdet) \leq \beta$ and robust risk with detection $R_{\text{adv-det}}^\epsilon (\fdet) \leq \beta$.
	
	The detector $\fdet$ is constructed as follows on a (possibly perturbed) input $\hat{x}$:
	\begin{itemize}
		\item Run the classifier $y \gets \clf(\hat{x})$.
		\item Find a perturbed input $x'$ withing distance $\epsilon/2$ of $\hat{x}$ that is classified differently, i.e., $d(\hat{x}, x') \leq \epsilon/2$ and $\clf(x') \neq y$. If such an input $x'$ exists, reject the input and output $\bot$. Else, output the class $y$.
	\end{itemize}
\end{theorem}

\begin{proof}[Proof of \Cref{thm:cls_to_det}]
	Note that for any input $(x, y)$ for which the classifier $\clf$ is robust at distance $\epsilon/2$, no input $x'$ above exists and so $\fdet(x) = y$. Thus, the risk of $\fdet$ is at most the robust risk of $\clf$, so $R(\fdet) \leq \beta$.
	
	Now, consider an input $(x, y) \sim \dist$ for which $\fdet$ is not robust with detection at distance $\epsilon$. That is, either $\fdet(x) \neq y $, or there exists an input $\hat{x}$ at distance $d(x, \hat{x}) \leq \epsilon$ such that $\fdet(\hat{x}) = \hat{y} \notin \{y, \bot\}$. We will show that the classifier $\clf$ is not robust for $x$ either (for attacks at distance up to $\epsilon/2$.) 
	
	If $\fdet(x) \neq y $, then by the same argument as above it cannot be the case that the classifier $\clf$ is robust at distance $\epsilon/2$ for $x$. 
	
	So let us consider the case where $\fdet(\hat{x}) = \hat{y} \notin \{y, \bot\}$.
	By the definition of $\fdet$, this means that for all $x'$ at distance at most $\epsilon/2$ from $\hat{x}$, we have $\clf(x') = \hat{y}$. But, note that there exists a point $x^*$ that is at distance at most $\epsilon/2$ from both $\hat{x}$ and $x$. Since we must have $\clf(x^*) = \hat{y}$, we conclude that $\clf$ is not robust at distance $\epsilon/2$ for $x$. 
	
	Taking expectations over the distribution $\dist$ concludes the proof.
\end{proof}

A main distinction between \Cref{thm:det_to_cls} and \Cref{thm:cls_to_det} is that the construction in \Cref{thm:det_to_cls} preserves clean accuracy, but the construction in \Cref{thm:cls_to_det} does not. The constructed robust detector in  \Cref{thm:cls_to_det} has \emph{clean} accuracy that is equal to the robust classifier's \emph{robust} accuracy.

The construction in \Cref{thm:cls_to_det} can be efficiently (but approximately) instantiated by a \emph{certifiably robust} classifier~\citep{wong2018provable, raghunathan2018certified}. These defenses can certify that a classifier's output is constant for all points within some distance $\epsilon$ of the input. For an adversarial example $\hat{x}$ for $\clf$, the certification always fails and thus the constructed detector $\fdet$ will reject $\hat{x}$. If $\clf$ is robust and the certification succeeds, the detector $\fdet$ copies the output of $\clf$. However, a certified defense may fail to certify a robust input (a false negative), and thus the detector $\fdet$ may reject more inputs than with the construction in \Cref{thm:cls_to_det}. This reduction from a certified classifier to a detector is implicit in~\citep[Section 3.1]{wong2018provable}.

\subsection{On Computational Complexity and Sample Complexity}
In statistical learning theory, it is customary to consider two types of \emph{complexity} for a learning algorithm: its \emph{sample complexity}, and its \emph{computational complexity}.

The sample complexity of a learning algorithm corresponds to the number of training samples that are necessary for attaining a given test error. In turn, the computational complexity represents the runtime of the learning algorithm.

Many classical results in statistical learning theory---such as standard generalization bounds---are only concerned with the sample complexity of a learning problem, and make no explicit assumptions on a learner's computational complexity.

Note that our reductions in \Cref{thm:det_to_cls} and \Cref{thm:cls_to_det} preserve the sample complexity of the original detector or classifier that the reduction is applied to. Indeed, the reduction uses a pre-trained detector/classifier as a black-box and only modifies the algorithm's inference phase.
Thus, from a sample complexity perspective, our results show that robust classification and detection are equally hard (up to a factor 2 in the perturbation bound).

\section{What Are Detection Defenses Claiming?}
\label{sec:experiments}


We now survey 14 detection defenses, and consider the robust \emph{classification} performance that these defenses implicitly claim (via \Cref{thm:det_to_cls}). As we will see, in 12/14 cases, the defenses' detection results imply a computationally inefficient classifier with far better robust accuracy than the state-of-the-art. 

Before presenting our experimental setup and the explicit results from the reduction, we first discuss how we believe these results should be interpreted.

\subsection{Interpreting our Reduction}

Suppose that some detector defense claims a robust accuracy that implies---via our reduction---an inefficient classifier with much higher robustness than the state-of-the-art (e.g., the defense A described in the introduction of this paper).

A first possible interpretation of our reduction is that this robust detector implies the \emph{existence} of a robust classifier. This interpretation is rather weak however, since it is typically presumed that robust classification is possible, and that human perception gives an existence proof of a robust classifier. The mere existence of a robust classifier is thus typically already assumed to be true.

\emph{Our reduction yields a stronger result. It gives an \textbf{explicit construction} of a sample efficient but computationally inefficient robust classifier from a robust detector.}

The question then is whether we should expect computationally inefficient robust classifiers to be easier to construct than efficient ones (at a fixed sample complexity). That is, do we expect that we can leverage computational inefficiency to build more robust classifiers than the current state-of-the-art, without having to collect additional training data?

We do not know of an answer to this question, but there is evidence to suggest that the answer may be negative.\footnote{Some works show that for certain ``unnatural'' distributions, computational inefficiency is necessary to build robust classifiers \citep{garg2020adversarially, bubeck2018adversarial}. Yet, since we presume that the human perceptual system \emph{is} robust to small perturbations on natural data (e.g., such as CIFAR-10), there must exist some efficient natural process to achieve robustness on such data.}
For example, the work of \citet{schmidt2018adversarially} proves that for a synthetic classification task between Gaussian distributions, building more robust classifiers requires additional \emph{data}---regardless of the amount of computation power.
Their results are corroborated by current state-of-the-art robust classifiers based on adversarial training~\citep{madry2018towards}, which do not appear to be limited by computational constraints. On CIFAR-10 for example, adversarial training achieves 100\% robust \emph{training} accuracy~\citep{schmidt2018adversarially}. 
Thus, it is unclear how computational inefficiency could be leveraged to build more robust classifiers using existing techniques.

Candidate approaches could be to train much larger models (e.g., with an exponential number of parameters), or to perform an exhaustive architecture search to find more robust models.
Yet, note that the robust classifier constructed in our reduction only uses its unbounded computational power at \emph{inference time}. That is, the classifier that is built in \Cref{thm:det_to_cls} uses a pre-trained detector model as a black-box subroutine (which is presumed to be efficient), and then solves a non-convex optimization problem at inference time.
The classifier built in our reduction is thus presumably weaker than a robust classifier that can be \emph{trained} with unbounded computational power.

To summarize, \textbf{when a detector defense claims a high robust accuracy, this implies the existence of a \emph{concretely instantiatable, sample efficient robust classifier with an inefficient inference procedure} with the same robust accuracy.}
 
If this inefficient classifier is much more robust than the current state-of-the-art, it thus means that either: (1) the detector defense inadvertently made a huge breakthrough in robust \emph{classification}; or (2) the detector defense's claims are incorrect...

Our results do not imply that the latter option is necessarily the right one. But given how challenging robust classification is proving to be, we have reason to be skeptical of such a major breakthrough (even for inefficient classifiers). To compound this, many proposed detection defenses are quite \emph{simple},\footnote{This simplicity is sometimes justified with the assumption that detecting adversarial examples is a much easier problem than classifying them. As we show, this assumption is flawed.} and reject adversarial inputs based on some standard statistical test over a neural network's features. It would thus be particularly surprising if such simple techniques could yield robust \emph{classifiers}, given that most ``simple'' approaches to adversarial robustness (denoising, compression, randomness, etc.) are known to be ineffective~\citep{he2017adversarial}. It is thus not too surprising that a number of the detector defenses that we survey hereafter have later been broken by stronger attacks~\citep{carlini2017adversarial, tramer2020adaptive}. Our reduction would have suggested that such a break was likely to happen.

Of course, one could always take the optimistic interpretation that the detector defense's claims \emph{are correct}, and that the defense does in fact imply a huge breakthrough in robust classification.
We then argue that detector defenses which claim such high robustness should also explicitly claim a breakthrough in robust classification (since our theorem proves that these claims are equivalent). In turn, these defenses should then be met with the scientific skepticism and scrutiny that such a breakthrough claim naturally entails.

\subsection{Experimental Setup}

We choose 14 detector defenses from the literature (see \Cref{tab:results}). Our selection of these defenses was partially motivated by a pragmatic consideration on the easiness of translating the defenses' claims into a bound on the robust risk with detection $R_{\text{adv-det}}^\epsilon$.
Indeed, some defenses simply report a single AUC score for the detector's performance, from which we cannot derive a useful bound on the robust risk.
We thus focus on defenses that either directly report a robust error akin to \Cref{def:risk_det}, or that provide concrete pairs of false-positive and false-negative rates (e.g., a full ROC curve). 
In the latter case, we compute a ``best-effort'' bound on the robust risk with detection\footnote{Many detector defenses report performance against a set of \emph{fixed} (non-adaptive) attacks. We interpret these results as being an approximation of the worst-case risk.}  as:
\begin{equation}
	\label{eq:bound}
R_{\text{adv-det}}^\epsilon(\fdet) \leq \text{FPR} + \text{FNR} + R(\fdet) \;,
\end{equation}

where FPR and FNR are the detector's false-positive and false-negative rates for a fixed detection threshold, and $R(\fdet)$ is the defense's standard risk (i.e., the test error on natural examples).

The above union bound in \Cref{eq:bound} is quite pessimistic, as we may over-count examples that lead to multiple sources of errors (e.g., a natural input that is misclassified and erroneously detected).
The true robustness claim made by these detector defenses might thus be stronger than what we obtain from our bound (note that this only strengthens our claims).
We encourage future detection papers to report their adversarial risk with detection, $R_{\text{adv-det}}^\epsilon$, to facilitate direct comparisons with robust classifiers using our reduction.

\begin{table*}[t]
	\centering
	\caption{For each detector defense, we compute a (best-effort) bound on the claimed robust risk with detection $R_{\text{adv-det}}^\epsilon$ using \Cref{eq:bound}, and report the complement (the robust accuracy with detection), $1-R_{\text{adv-det}}^\epsilon$.  For each detector's robustness claim (at distance $\epsilon$), we report the state-of-the-art robust classification accuracy  for attacks at distance $\epsilon/2$, denoted $1-R_{\text{adv}}^{\epsilon/2}$.
		Detection defense claims that imply a higher robust classification accuracy than the current state-of-the-art are {\color{red} highlighted in red}.}
	\vspace{1em}
	\begin{tabular}{@{} l l l r c H c@{}}
		\textbf{Dataset} & \textbf{Defense} & \textbf{Norm} & $\bm{\epsilon}$ & $\bm{1-R_{\textbf{adv-det}}^\epsilon}$ & $\bm{\epsilon/2}$  
		& $\bm{1-R_{\textbf{adv}}^{\epsilon/2}}$ \\
		\toprule
		\multirow{3.5}{5em}{MNIST} 
		& \citet{grosse2017statistical} & $\ell_\infty$ & $0.5$ & \color{red} $\geq 98\%$ & $0.25$ & $94\%$ \\
		\cmidrule{2-7}
		& \citet{ma2018characterizing} & $\ell_2$ & $4.2$ & \color{red} $\geq 99\%$ & $2.1$ & $72\%$ \\
		& \citet{raghuram2021general} & $\ell_2$ & $8.9$ & \color{red} $\geq 74\%$ & $4.5$ & $0\%$ \\
		\midrule
		\multirow{8.5}{5em}{CIFAR-10} 
		& \citet{Yin2020GAT} &  $\ell_2$ & $1.7$ & \color{red} $\geq90\%$ & $0.85$ & $66\%$ \\
		& \citet{feinman2017detecting} & $\ell_2$ & $2.7$ & \color{red}  $\geq43\%$ & $1.35$ & $36\%$ \\
		& \citet{miller2019not} &  $\ell_2$ & $2.9$ & \color{red} $\geq75\%$ & $1.45$ & $30\%$ \\
		& \citet{raghuram2021general} &  $\ell_2$ & $4.0$ & \color{red} $\geq56\%$ & $2.0$ & $10\%$ \\
		\cmidrule{2-7}
		& \citet{ma2019nic} &  $\ell_\infty$ & $\sfrac{4}{255}$ & \color{red} $\geq96\%$ & $\sfrac{2}{255}$ & $85\%$ \\
		& \citet{roth2019odds} & $\ell_\infty$ & $\sfrac{8}{255}$ & $\geq66\%$ & $\sfrac{4}{255}$ & $79\%$ \\
		& \citet{lee2018simple} &  $\ell_\infty$ & $\sfrac{20}{255}$ & \color{red} $\geq81\%$ & $\sfrac{10}{255}$ & $59\%$ \\
		& \citet{li2019generative} &  $\ell_\infty$ & $\sfrac{26}{255}$ & \color{red} $\geq80\%$ & $\sfrac{13}{255}$ & $44\%$ \\
		\midrule
		\multirow{5.5}{5em}{ImageNet} 
		& \citet{xu2017feature} &  $\ell_2$ & $1.0$ & \color{red} $\geq67\%$ & $0.5$ & $54\%$ \\
		\cmidrule{2-7}
		& \citet{ma2019nic} &  $\ell_\infty$ & $\sfrac{2}{255}$  & \color{red} $\geq68\%$ & $\sfrac{1}{255}$ & $55\%$ \\
		& \citet{jha2019attribution} &  $\ell_\infty$ & $\sfrac{2}{255}$  &  $\geq30\%$ & $\sfrac{1}{255}$ & $55\%$ \\
		& \citet{hendrycks2016early} & $\ell_\infty$ & $\sfrac{10}{255}$ & \color{red} $\geq76\%$ & $\sfrac{5}{255}$ & $30\%$ \\
		& \citet{yu2019new} &  $\ell_\infty$ & $\sfrac{26}{255}$ & \color{red} $\geq\hphantom{0}7\%$ & $\sfrac{13}{255}$ & $\hphantom{0}5\%$ \\
		\bottomrule	
	\end{tabular}
	\label{tab:results}
\end{table*}

The 14 detector defenses use three datasets: MNIST, CIFAR-10 and ImageNet, and consider adversarial examples under the $\ell_\infty$ or $\ell_2$ norms. Given a claim of robust detection at distance $\epsilon$, we contrast it to a state-of-the-art robust classification result for distance $\epsilon/2$:

\begin{itemize}
\item On MNIST with $\ell_\infty$ attacks, we use the adversarially-trained TRADES classifier~\citep{zhang2019theoretically} and measure robust error with the Square attack~\citep{andriushchenko2020square}.
\item On MNIST with $\ell_2$ attacks, we use the adversarially-trained classifier from~\citet{tramer2019adversarial} and measure robust error with PGD~\citep{madry2018towards}.
\item On CIFAR-10, for both $\ell_\infty$ and $\ell_2$ attacks we use the adversarially-trained classifier of~\citet{rebuffi2021fixing} (trained without external data), and attack it using the APGD-CE attack from AutoAttack~\citep{croce2020reliable}.
\item For ImageNet, for both $\ell_\infty$ and $\ell_2$ attacks we use adversarially-trained classifiers and PGD attacks from~\citet{robustness}.
\end{itemize}

We also consider two \emph{certified} defenses for $\ell_\infty$ attacks on CIFAR-10: the robust classifier
of~\citet{zhang2019towards}, and a recent certified detector of~\citet{sheikholeslami2021provably}.

\subsection{Results}

As we can see from \Cref{tab:results},  most defenses claim a detection performance that implies a far greater robust accuracy than our current best robust classifiers.
To illustrate with a concrete example, the CIFAR-10 detector of \citet{miller2019not} claims to achieve robust accuracy with detection of $75\%$ for $\ell_2$ attacks with $\epsilon=2.9$. Using \Cref{thm:det_to_cls}, this implies an inefficient classifier with robust accuracy of $75\%$ for $\ell_2$ attacks with $\epsilon=\sfrac{2.9}{2}=1.45$. Yet, the current state-of-the-art robust accuracy for such a perturbation budget is only $30\%$~\citep{rebuffi2021fixing}.
If this detector defense's robustness claim were correct, it would imply a remarkable breakthrough in robust classification!

Why do many of these defenses claim robust accuracies that appear far beyond the current state-of-the-art? A primary reason is that the vast majority of the above detector defenses do not consider evaluations against \emph{adaptive attacks}~\citep{carlini2019evaluating, athalye2018obfuscated, tramer2020adaptive}. That is, these defenses show that they can detect \emph{some fixed attacks}, and thereafter conclude that the detector is robust against \emph{all attacks}. As in the case of robust classifiers, such an evaluation is clearly insufficient.
Some defenses do evaluate against adaptive adversaries, but likely fail to build a sufficiently strong attack to reliably approximate the worst-case robust risk. Because of the lack of a strong comparative baseline, it is not always immediately clear that these results are overly strong.

For example, the recent work of \citet[ICML Long Talk]{raghuram2021general}  builds a detector on MNIST with a FNR of $\leq 5\%$ at a FPR of $\leq 20\%$, for \emph{adaptive} $\ell_2$ attacks bounded by $\epsilon=8.9$. Yet, this perturbation bound is much larger than the average distance between an MNIST image and the nearest image from a different class! Thus, an attack within this perturbation bound can trivially reduce the detector's accuracy to chance.
On CIFAR-10, the same detector achieves $95\%$ clean accuracy, and a FNR of $\leq 19\%$  at a FPR of $\leq 20\%$ for \emph{adaptive} $\ell_2$ attacks bounded by $\epsilon=4$. Using \Cref{eq:bound}, this yields a bound on the robust accuracy with detection of $1-R_{\text{adv-det}}^\epsilon(\fdet) \geq 1-(5\% + 19\% + 20\%) = 56\%$. In contrast, the best robust classifier we are aware of for $\ell_2$ attacks bounded by $\epsilon=2$ achieves robust accuracy of only $10\%$~\citep{rebuffi2021fixing}.
In summary, the adaptive attack considered in this detector defense's evaluation is highly unlikely to be a good approximation of a worst-case attack, and this defense can probably be broken by stronger attacks (or if not, the defense should deservedly claim to have made a huge breakthrough in robust classification!).

\paragraph{Certified classifiers and detectors.}
In \Cref{tab:certified}, we look at the robust accuracy with detection and standard robust accuracy achieved by \emph{certified} defenses (for which the claimed robustness numbers are necessarily mathematically correct). 

We note that our reduction is not as meaningful in the case of certified defenses, since it is highly plausible that computational inefficiency \emph{can} be leveraged to build better certified classifiers.
Indeed, given any robust classifier (e.g., an adversarially trained model), the classifier's robustness can always be certified inefficiently (by enumerating over all points within an $\epsilon$-ball).
Thus, the existence of an inefficient classifier with higher \emph{certified} robustness than the state-of-the-art would not be particularly surprising.

Nevertheless, we find that existing results for certified classifiers and detectors perfectly match what is implied by our reduction (up to $\pm 2\%$ error). For example, \citet{zhang2019towards} achieve $39\%$ certified robust accuracy on CIFAR-10 for perturbations of $\ell_\infty$-norm below $\sfrac{4}{255}$. Together with \Cref{thm:cls_to_det}, this implies an inefficient detector with $39\%$ robust detection accuracy for perturbations of $\ell_\infty$-norm below $\sfrac{8}{255}$. The recent work of \citet{sheikholeslami2021provably} nearly matches that bound ($37\%$ robust accuracy with detection), with a defense that has the advantage of being concretely efficient.

These results give additional credence to our thesis: with current techniques, robust classification is indeed approximately twice as hard (in terms of the perturbation bounds covered) than robust detection.

\begin{table}[t]
	\centering
	\renewcommand{\arraystretch}{1}
	\caption{Certified robust accuracy $1-R_{\text{adv}}^{\epsilon/2}$ for the defense of \citet{zhang2019towards}, and certified robust accuracy with detection  $1-R_{\text{adv-det}}^{\epsilon}$ for the defense of \citet{sheikholeslami2021provably}.}
	\vspace{1em}
	\begin{tabular}{@{} c c H c@{}}
		$\bm{\epsilon}$ & $\bm{1-R_{\textbf{adv-det}}^\epsilon}$ & $\bm{\epsilon/2}$  & $\bm{1-R_{\textbf{adv}}^{ \epsilon/2}}$ \\
		\toprule
		$\sfrac{\hphantom{0}8}{255}$ & $37\%$ & $\sfrac{4}{255}$ & $39\%$ \\
		$\sfrac{16}{255}$ & $32\%$ & $\sfrac{8}{255}$ & $33\%$ \\
		\bottomrule
	\end{tabular}
	\label{tab:certified}
\end{table}

\section{Extensions and Open Problems}
The main open problem raised by our work is whether it is possible to show an \emph{efficient} reduction between classification and detection of adversarial examples, but this seems implausible---at least with our minimum distance decoding approach.

Another interesting question is whether a similar reduction can be shown for robustness to less structured perturbations than $\ell_p$ balls and other metric spaces. For example, there has been a line of research on defending against \emph{adversarial patches}~\citep{brown2017adversarial}, using empirical \citep{hayes2018visible, naseer2019local, chou2018sentinet} and certifiable techniques \citep{chiang2020certified, zhang2020clipped, xiang2021patchguard}. To use our result, we would have to define some metric to measure the size of an adversarial patch's perturbation. Yet, the size of a patch is typically defined by the number of contiguously perturbed pixels, which does not define a metric---in particular, it does not satisfy the triangular inequality which is necessary in our proofs.

Finally, similar hardness reductions as ours might exist between other approaches for building robust classifiers. For example, the question of whether (test-time) randomness can be leveraged to build more robust models is also intriguing. Empirical defenses that use randomness can be notoriously hard to evaluate~\citep{athalye2018obfuscated, tramer2020adaptive}, so a reduction similar to ours might be useful in showing that we should not expect such approaches to bare fruit.


\section{Conclusion}

We have shown formal reductions between robust classification with, and without, a detection option.
Our results show that significant progress on one of these two tasks implies similar progress on the other---unless computational inefficiency can somehow be leveraged to build more robust models. 
This raises the question on whether we should spend our efforts on studying both of these tasks, or focus our efforts on a single one.

On one hand, the two tasks represent different ways of tackling a common goal, and working on either task might result in new techniques or ideas that apply to the other task as well.
On the other hand, our reductions show that unless we make progress on both tasks, work on one of the tasks can merely aim to match the robustness of our inefficient constructions, whilst improving their computational complexity.

We believe our reduction will serve as a useful sanity-check when assessing the claims of future detector defenses. Detector defenses' robustness evaluations have received less stringent scrutiny than robust classifiers over the past years, perhaps in part due to a lack of strong comparative baselines. Instead of having to wait until some detector defense's claims pass the test-of-time, we show that detection results can be directly contrasted against long-standing results for robust classification.

When applying this approach to past detector defenses, we find that many make robustness claims that imply significant breakthroughs in robust classification. We believe our reduction could have been useful in highlighting the suspiciously strong claims made by many of these defenses---before they were explicitly broken by stronger attacks.


\section*{Acknowledgments}
Thanks to thank Nicholas Carlini and Wieland Brendel for helpful discussions, as well as to Alex Ozdemir for suggesting the connection between our results and minimum distance decoding.



\bibliography{biblio}

\begin{thebibliography}{49}
\providecommand{\natexlab}[1]{#1}
\providecommand{\url}[1]{\texttt{#1}}
\expandafter\ifx\csname urlstyle\endcsname\relax
  \providecommand{\doi}[1]{doi: #1}\else
  \providecommand{\doi}{doi: \begingroup \urlstyle{rm}\Url}\fi

\bibitem[Aaronson(2008)]{aaronson}
Aaronson, S.
\newblock Ten signs a claimed mathematical breakthrough is wrong.
\newblock \url{https://www.scottaaronson.com/blog/?p=304}, Jan 2008.

\bibitem[Andriushchenko et~al.(2020)Andriushchenko, Croce, Flammarion, and
  Hein]{andriushchenko2020square}
Andriushchenko, M., Croce, F., Flammarion, N., and Hein, M.
\newblock Square attack: a query-efficient black-box adversarial attack via
  random search.
\newblock In \emph{European Conference on Computer Vision}, 2020.

\bibitem[Athalye et~al.(2018)Athalye, Carlini, and
  Wagner]{athalye2018obfuscated}
Athalye, A., Carlini, N., and Wagner, D.
\newblock Obfuscated gradients give a false sense of security: Circumventing
  defenses to adversarial examples.
\newblock In \emph{International Conference on Machine Learning}, 2018.

\bibitem[Biggio et~al.(2013)Biggio, Corona, Maiorca, Nelson, {\v{S}}rndi{\'c},
  Laskov, Giacinto, and Roli]{biggio2013evasion}
Biggio, B., Corona, I., Maiorca, D., Nelson, B., {\v{S}}rndi{\'c}, N., Laskov,
  P., Giacinto, G., and Roli, F.
\newblock Evasion attacks against machine learning at test time.
\newblock In \emph{European Conference on Machine Learning and Knowledge
  Discovery in Databases}, pp.\  387--402. Springer, 2013.

\bibitem[Brown et~al.(2017)Brown, Man{\'e}, Roy, Abadi, and
  Gilmer]{brown2017adversarial}
Brown, T.~B., Man{\'e}, D., Roy, A., Abadi, M., and Gilmer, J.
\newblock Adversarial patch.
\newblock \emph{arXiv preprint arXiv:1712.09665}, 2017.

\bibitem[Bubeck et~al.(2018)Bubeck, Lee, Price, and
  Razenshteyn]{bubeck2018adversarial}
Bubeck, S., Lee, Y.~T., Price, E., and Razenshteyn, I.
\newblock Adversarial examples from cryptographic pseudo-random generators.
\newblock \emph{arXiv preprint arXiv:1811.06418}, 2018.

\bibitem[Carlini \& Wagner(2017)Carlini and Wagner]{carlini2017adversarial}
Carlini, N. and Wagner, D.
\newblock Adversarial examples are not easily detected: Bypassing ten detection
  methods.
\newblock In \emph{AISec}, pp.\  3--14. ACM, 2017.

\bibitem[Carlini et~al.(2019)Carlini, Athalye, Papernot, Brendel, Rauber,
  Tsipras, Goodfellow, and Madry]{carlini2019evaluating}
Carlini, N., Athalye, A., Papernot, N., Brendel, W., Rauber, J., Tsipras, D.,
  Goodfellow, I., and Madry, A.
\newblock On evaluating adversarial robustness.
\newblock \emph{arXiv preprint arXiv:1902.06705}, 2019.

\bibitem[Carmon et~al.(2019)Carmon, Raghunathan, Schmidt, Duchi, and
  Liang]{carmon2019unlabeled}
Carmon, Y., Raghunathan, A., Schmidt, L., Duchi, J.~C., and Liang, P.~S.
\newblock Unlabeled data improves adversarial robustness.
\newblock In \emph{Advances in Neural Information Processing Systems}, pp.\
  11192--11203, 2019.

\bibitem[Chiang et~al.(2020)Chiang, Ni, Abdelkader, Zhu, Studer, and
  Goldstein]{chiang2020certified}
Chiang, P.-y., Ni, R., Abdelkader, A., Zhu, C., Studer, C., and Goldstein, T.
\newblock Certified defenses for adversarial patches.
\newblock \emph{arXiv preprint arXiv:2003.06693}, 2020.

\bibitem[Chou et~al.(2020)Chou, Tram{\`e}r, and Pellegrino]{chou2018sentinet}
Chou, E., Tram{\`e}r, F., and Pellegrino, G.
\newblock Sentinet: Detecting physical attacks against deep learning systems.
\newblock In \emph{Workshop on Deep Learning Security}, 2020.

\bibitem[Croce \& Hein(2020)Croce and Hein]{croce2020reliable}
Croce, F. and Hein, M.
\newblock Reliable evaluation of adversarial robustness with an ensemble of
  diverse parameter-free attacks.
\newblock In \emph{ICML}, 2020.

\bibitem[Croce et~al.(2020)Croce, Andriushchenko, Sehwag, Flammarion, Chiang,
  Mittal, and Hein]{croce2020robustbench}
Croce, F., Andriushchenko, M., Sehwag, V., Flammarion, N., Chiang, M., Mittal,
  P., and Hein, M.
\newblock Robustbench: a standardized adversarial robustness benchmark.
\newblock \emph{arXiv preprint arXiv:2010.09670}, 2020.

\bibitem[Engstrom et~al.(2019)Engstrom, Ilyas, Salman, Santurkar, and
  Tsipras]{robustness}
Engstrom, L., Ilyas, A., Salman, H., Santurkar, S., and Tsipras, D.
\newblock Robustness (python library), 2019.
\newblock URL \url{https://github.com/MadryLab/robustness}.

\bibitem[Feinman et~al.(2017)Feinman, Curtin, Shintre, and
  Gardner]{feinman2017detecting}
Feinman, R., Curtin, R.~R., Shintre, S., and Gardner, A.~B.
\newblock Detecting adversarial samples from artifacts.
\newblock \emph{arXiv preprint arXiv:1703.00410}, 2017.

\bibitem[Garg et~al.(2020)Garg, Jha, Mahloujifar, and
  Mohammad]{garg2020adversarially}
Garg, S., Jha, S., Mahloujifar, S., and Mohammad, M.
\newblock Adversarially robust learning could leverage computational hardness.
\newblock In \emph{Algorithmic Learning Theory}, pp.\  364--385. PMLR, 2020.

\bibitem[Grosse et~al.(2017)Grosse, Manoharan, Papernot, Backes, and
  McDaniel]{grosse2017statistical}
Grosse, K., Manoharan, P., Papernot, N., Backes, M., and McDaniel, P.
\newblock On the (statistical) detection of adversarial examples.
\newblock \emph{arXiv preprint arXiv:1702.06280}, 2017.

\bibitem[Hayes(2018)]{hayes2018visible}
Hayes, J.
\newblock On visible adversarial perturbations \& digital watermarking.
\newblock In \emph{Proceedings of the IEEE Conference on Computer Vision and
  Pattern Recognition Workshops}, pp.\  1597--1604, 2018.

\bibitem[He et~al.(2017)He, Wei, Chen, Carlini, and Song]{he2017adversarial}
He, W., Wei, J., Chen, X., Carlini, N., and Song, D.
\newblock Adversarial example defenses: Ensembles of weak defenses are not
  strong.
\newblock In \emph{USENIX Workshop on Offensive Technologies}, 2017.

\bibitem[Hendrycks \& Gimpel(2017)Hendrycks and Gimpel]{hendrycks2016early}
Hendrycks, D. and Gimpel, K.
\newblock Early methods for detecting adversarial images.
\newblock In \emph{International Conference on Learning Representations}, 2017.

\bibitem[Jha et~al.(2019)Jha, Raj, Fernandes, Jha, Jha, Verma, Jalaian, and
  Swami]{jha2019attribution}
Jha, S., Raj, S., Fernandes, S.~L., Jha, S.~K., Jha, S., Verma, G., Jalaian,
  B., and Swami, A.
\newblock Attribution-driven causal analysis for detection of adversarial
  examples.
\newblock \emph{arXiv preprint arXiv:1903.05821}, 2019.

\bibitem[Lee et~al.(2018)Lee, Lee, Lee, and Shin]{lee2018simple}
Lee, K., Lee, K., Lee, H., and Shin, J.
\newblock A simple unified framework for detecting out-of-distribution samples
  and adversarial attacks.
\newblock In \emph{Advances in Neural Information Processing Systems}, 2018.

\bibitem[Li et~al.(2019)Li, Bradshaw, and Sharma]{li2019generative}
Li, Y., Bradshaw, J., and Sharma, Y.
\newblock Are generative classifiers more robust to adversarial attacks?
\newblock In \emph{International Conference on Machine Learning}, pp.\
  3804--3814. PMLR, 2019.

\bibitem[Ma \& Liu(2019)Ma and Liu]{ma2019nic}
Ma, S. and Liu, Y.
\newblock Nic: Detecting adversarial samples with neural network invariant
  checking.
\newblock In \emph{Proceedings of the 26th Network and Distributed System
  Security Symposium (NDSS 2019)}, 2019.

\bibitem[Ma et~al.(2018)Ma, Li, Wang, Erfani, Wijewickrema, Schoenebeck, Song,
  Houle, and Bailey]{ma2018characterizing}
Ma, X., Li, B., Wang, Y., Erfani, S.~M., Wijewickrema, S., Schoenebeck, G.,
  Song, D., Houle, M.~E., and Bailey, J.
\newblock Characterizing adversarial subspaces using local intrinsic
  dimensionality.
\newblock In \emph{International Conference on Learning Representations}, 2018.

\bibitem[Madry et~al.(2018)Madry, Makelov, Schmidt, Tsipras, and
  Vladu]{madry2018towards}
Madry, A., Makelov, A., Schmidt, L., Tsipras, D., and Vladu, A.
\newblock Towards deep learning models resistant to adversarial attacks.
\newblock In \emph{International Conference on Learning Representations}, 2018.

\bibitem[Miller et~al.(2019)Miller, Wang, and Kesidis]{miller2019not}
Miller, D., Wang, Y., and Kesidis, G.
\newblock When not to classify: Anomaly detection of attacks (ada) on dnn
  classifiers at test time.
\newblock \emph{Neural computation}, 31\penalty0 (8):\penalty0 1624--1670,
  2019.

\bibitem[Naseer et~al.(2019)Naseer, Khan, and Porikli]{naseer2019local}
Naseer, M., Khan, S., and Porikli, F.
\newblock Local gradients smoothing: Defense against localized adversarial
  attacks.
\newblock In \emph{2019 IEEE Winter Conference on Applications of Computer
  Vision (WACV)}, pp.\  1300--1307. IEEE, 2019.

\bibitem[Pang et~al.(2021)Pang, Zhang, He, Dong, Su, Chen, Zhu, and
  Liu]{pang2021adversarial}
Pang, T., Zhang, H., He, D., Dong, Y., Su, H., Chen, W., Zhu, J., and Liu,
  T.-Y.
\newblock Adversarial training with rectified rejection.
\newblock \emph{arXiv preprint arXiv:2105.14785}, 2021.

\bibitem[Raghunathan et~al.(2018)Raghunathan, Steinhardt, and
  Liang]{raghunathan2018certified}
Raghunathan, A., Steinhardt, J., and Liang, P.
\newblock Certified defenses against adversarial examples.
\newblock In \emph{International Conference on Learning Representations}, 2018.

\bibitem[Raghuram et~al.(2021)Raghuram, Chandrasekaran, Jha, and
  Banerjee]{raghuram2021general}
Raghuram, J., Chandrasekaran, V., Jha, S., and Banerjee, S.
\newblock A general framework for detecting anomalous inputs to dnn
  classifiers.
\newblock In \emph{International Conference on Machine Learning}, 2021.

\bibitem[Rebuffi et~al.(2021)Rebuffi, Gowal, Calian, Stimberg, Wiles, and
  Mann]{rebuffi2021fixing}
Rebuffi, S.-A., Gowal, S., Calian, D.~A., Stimberg, F., Wiles, O., and Mann, T.
\newblock Fixing data augmentation to improve adversarial robustness.
\newblock \emph{arXiv preprint arXiv:2103.01946}, 2021.

\bibitem[Roth et~al.(2019)Roth, Kilcher, and Hofmann]{roth2019odds}
Roth, K., Kilcher, Y., and Hofmann, T.
\newblock The odds are odd: A statistical test for detecting adversarial
  examples.
\newblock In \emph{International Conference on Machine Learning}, 2019.

\bibitem[Schmidt et~al.(2018)Schmidt, Santurkar, Tsipras, Talwar, and
  Madry]{schmidt2018adversarially}
Schmidt, L., Santurkar, S., Tsipras, D., Talwar, K., and Madry, A.
\newblock Adversarially robust generalization requires more data.
\newblock In \emph{Advances In Neural Information Processing Systems}, pp.\
  5019--5031, 2018.

\bibitem[Sheikholeslami et~al.(2021)Sheikholeslami, Lotfi, and
  Kolter]{sheikholeslami2021provably}
Sheikholeslami, F., Lotfi, A., and Kolter, J.~Z.
\newblock Provably robust classification of adversarial examples with
  detection.
\newblock In \emph{International Conference on Learning Representations}, 2021.

\bibitem[Szegedy et~al.(2014)Szegedy, Zaremba, Sutskever, Bruna, Erhan,
  Goodfellow, and Fergus]{szegedy2013intriguing}
Szegedy, C., Zaremba, W., Sutskever, I., Bruna, J., Erhan, D., Goodfellow, I.,
  and Fergus, R.
\newblock Intriguing properties of neural networks.
\newblock In \emph{International Conference on Learning Representations}, 2014.

\bibitem[Tram{\`e}r \& Boneh(2019)Tram{\`e}r and Boneh]{tramer2019adversarial}
Tram{\`e}r, F. and Boneh, D.
\newblock Adversarial training and robustness for multiple perturbations.
\newblock In \emph{Advances In Neural Information Processing Systems}, 2019.

\bibitem[Tram{\`e}r et~al.(2020)Tram{\`e}r, Carlini, Brendel, and
  Madry]{tramer2020adaptive}
Tram{\`e}r, F., Carlini, N., Brendel, W., and Madry, A.
\newblock On adaptive attacks to adversarial example defenses.
\newblock In \emph{Conference on Neural Information Processing Systems}, 2020.

\bibitem[Tsipras et~al.(2019)Tsipras, Santurkar, Engstrom, Turner, and
  Madry]{tsipras2019robustness}
Tsipras, D., Santurkar, S., Engstrom, L., Turner, A., and Madry, A.
\newblock Robustness may be at odds with accuracy.
\newblock In \emph{International Conference on Learning Representations}, 2019.

\bibitem[Uesato et~al.(2019)Uesato, Alayrac, Huang, Stanforth, Fawzi, and
  Kohli]{uesato2019labels}
Uesato, J., Alayrac, J.-B., Huang, P.-S., Stanforth, R., Fawzi, A., and Kohli,
  P.
\newblock Are labels required for improving adversarial robustness?
\newblock \emph{arXiv preprint arXiv:1905.13725}, 2019.

\bibitem[Wong \& Kolter(2018)Wong and Kolter]{wong2018provable}
Wong, E. and Kolter, Z.
\newblock Provable defenses against adversarial examples via the convex outer
  adversarial polytope.
\newblock In \emph{International Conference on Machine Learning}, pp.\
  5283--5292, 2018.

\bibitem[Xiang et~al.(2021)Xiang, Bhagoji, Sehwag, and
  Mittal]{xiang2021patchguard}
Xiang, C., Bhagoji, A.~N., Sehwag, V., and Mittal, P.
\newblock Patchguard: A provably robust defense against adversarial patches via
  small receptive fields and masking.
\newblock In \emph{30th $\{$USENIX$\}$ Security Symposium ($\{$USENIX$\}$
  Security 21)}, 2021.

\bibitem[Xu et~al.(2018)Xu, Evans, and Qi]{xu2017feature}
Xu, W., Evans, D., and Qi, Y.
\newblock Feature squeezing: Detecting adversarial examples in deep neural
  networks.
\newblock In \emph{Network and Distributed System Security Symposium}, 2018.

\bibitem[Yin et~al.(2020)Yin, Kolouri, and Rohde]{Yin2020GAT}
Yin, X., Kolouri, S., and Rohde, G.~K.
\newblock Gat: Generative adversarial training for adversarial example
  detection and robust classification.
\newblock In \emph{International Conference on Learning Representations}, 2020.

\bibitem[Yu et~al.(2019)Yu, Hu, Guo, Chao, and Weinberger]{yu2019new}
Yu, T., Hu, S., Guo, C., Chao, W.-L., and Weinberger, K.~Q.
\newblock A new defense against adversarial images: Turning a weakness into a
  strength.
\newblock In \emph{Advances in Neural Information Processing Systems}, 2019.

\bibitem[Zhai et~al.(2019)Zhai, Cai, He, Dan, He, Hopcroft, and
  Wang]{zhai2019adversarially}
Zhai, R., Cai, T., He, D., Dan, C., He, K., Hopcroft, J., and Wang, L.
\newblock Adversarially robust generalization just requires more unlabeled
  data.
\newblock \emph{arXiv preprint arXiv:1906.00555}, 2019.

\bibitem[Zhang et~al.(2019)Zhang, Yu, Jiao, Xing, Ghaoui, and
  Jordan]{zhang2019theoretically}
Zhang, H., Yu, Y., Jiao, J., Xing, E.~P., Ghaoui, L.~E., and Jordan, M.~I.
\newblock Theoretically principled trade-off between robustness and accuracy.
\newblock In \emph{International Conference on Machine Learning}, 2019.

\bibitem[Zhang et~al.(2020{\natexlab{a}})Zhang, Chen, Xiao, Li, Boning, and
  Hsieh]{zhang2019towards}
Zhang, H., Chen, H., Xiao, C., Li, B., Boning, D., and Hsieh, C.-J.
\newblock Towards stable and efficient training of verifiably robust neural
  networks.
\newblock In \emph{International Conference on Learning Representations},
  2020{\natexlab{a}}.

\bibitem[Zhang et~al.(2020{\natexlab{b}})Zhang, Yuan, McCoyd, and
  Wagner]{zhang2020clipped}
Zhang, Z., Yuan, B., McCoyd, M., and Wagner, D.
\newblock Clipped bagnet: Defending against sticker attacks with clipped
  bag-of-features.
\newblock In \emph{2020 IEEE Security and Privacy Workshops (SPW)}, pp.\
  55--61. IEEE, 2020{\natexlab{b}}.

\end{thebibliography}
\bibliographystyle{icml2022}

\end{document}